%% file: main.tex
\documentclass[cameraready]{Interspeech}

\title{AGRI-Fidelity: Evaluating the Reliability of Listenable Explanations for Poultry Disease Detection}

\author[affiliation={}]{Sindhuja}{Madabushi}
\author[affiliation={}]{Arda}{Dogan}
\author[affiliation={}]{Jonathan}{Liu}
\author[affiliation={}]{Dian}{Chen}
\author[affiliation={}]{Dong S.}{Ha}
\author[affiliation={}]{Sook}{Shin}
\author[affiliation={}]{Sam H.}{Noh}
\author[affiliation={}]{Jin-Hee}{Cho}

\address{
Virginia Tech 
}

\email{\{msindhuja, arda, jliu27, dianc, dha, sook, samhnoh, jicho\}@vt.edu}


\usepackage{comment}
\usepackage{algorithm}
\usepackage{algpseudocode}
\usepackage{amsthm}
\newtheorem{theorem}{Theorem}
\usepackage{graphicx}
\usepackage{subcaption}
\usepackage{tikz}
\usepackage{multirow} 

\algnewcommand\algorithmicinput{\textbf{Input:}}
\algnewcommand\algorithmicoutput{\textbf{Output:}}
\algnewcommand\Input{\item[\algorithmicinput]}
\algnewcommand\Output{\item[\algorithmicoutput]}

\begin{document}

\maketitle
\begin{abstract}
Existing XAI metrics measure faithfulness for a single model, ignoring model multiplicity where near-optimal classifiers rely on different or spurious acoustic cues. In noisy farm environments, stationary artifacts such as ventilation noise can produce explanations that are faithful yet unreliable, as masking-based metrics fail to penalize redundant shortcuts. We propose \texttt{AGRI-Fidelity}, a reliability-oriented evaluation framework for listenable explanations in poultry disease detection without spatial ground truth. The method combines cross-model consensus with cyclic temporal permutation to construct null distributions and compute a False Discovery Rate (FDR), suppressing stationary artifacts while preserving time-localized bioacoustic markers. Across real and controlled datasets, \texttt{AGRI-Fidelity} effectively provides reliability-aware discrimination for all data points versus masking-based metrics.
\end{abstract}

\section{Introduction}
\label{sec:introduction}
Animal disease detection is critical for livestock health monitoring and timely intervention. With the increasing availability of poultry audio from commercial farms, deep learning models have proven effective at identifying disease-related vocalization patterns \cite{yajie2023poultry,vidic2017advanced}. However, despite strong predictive performance, these models remain black boxes, limiting adoption in agricultural settings where expert validation and trust are essential \cite{handelman2019peering}.

Explainable AI (XAI) methods address this limitation by identifying the acoustic features driving individual predictions. In audio tasks, listenable explanations enhance interpretability by enabling stakeholders to directly hear the time–frequency regions deemed important \cite{paissan2024listenable,wullenweber2022coughlime,parekh2022listen}. However, in noisy farm environments, explanations that are faithful to a model may still lack reliability. Continuous background artifacts such as ventilation hum or machinery noise are persistent, allowing models to exploit them as predictive shortcuts. XAI methods may then faithfully highlight these domain-irrelevant signals, creating a false sense of interpretability. Moreover, near-optimal models can depend on different acoustic cues, producing inconsistent explanations for the same input and further weakening trust.

\noindent \textbf{Why Are Existing XAI Metrics Insufficient for Reliable Evaluation?}
Existing XAI evaluation metrics~\cite{fresz2024classification,banerjee2022methods}, including faithfulness~\cite{li2023mathcal} and fidelity-based variants~\cite{miro2025comprehensive}, primarily assess whether an explanation aligns with a model’s internal decision process. However, they do not quantify whether the reliance is task-relevant or driven by spurious correlations. Consequently, these metrics can assign high scores to explanations that faithfully reflect model behavior yet provide little assurance of reliability or domain validity. This gap is particularly problematic in safety-critical applications, where stakeholders require principled guarantees about explanation trustworthiness.

In bioacoustic animal disease detection, models may achieve high predictive accuracy by exploiting confounding factors such as persistent background noise, microphone placement, recording equipment characteristics, or time-of-day effects, rather than physiologically meaningful vocal patterns. Because standard metrics assess internal consistency rather than statistical validity, they may assign high scores even when explanations highlight domain-irrelevant artifacts \cite{agrawal2026barriers}. Moreover, explanations can differ across near-optimal models, exposing instability that single-model metrics fail to capture \cite{marx2023but}.

Compounding this issue, obtaining ground-truth feature-importance annotations is challenging in real-world farm environments. Distance- or overlap-based metrics that assume precise spatial or temporal masks are therefore impractical for noisy, unconstrained bioacoustic data \cite{stodt2023novel,kadir2023evaluation}. These limitations underscore the need for an evaluation framework that moves beyond single-model faithfulness toward statistically grounded cross-model reliability assessment.

\noindent \textbf{Our Approach.} To address these gaps, we propose \texttt{AGRI-Fidelity}, a semantic reliance fidelity metric for evaluating listenable explanations in audio-based poultry disease detection systems. Unlike existing fidelity measures, \texttt{AGRI-Fidelity} augments standard fidelity with an explicit reliability score, quantifying the extent to which a model’s explanations can be trusted. 
\texttt{AGRI-Fidelity} computes a consensus-based stability score across multiple model families to assess explanation consistency, to identify physiologically meaningful signals against stationary artifacts, and integrates stability and fidelity into a single reliability score.

\noindent \textbf{Key Contributions.}
\textbf{(1)} We introduce a \textbf{reliability-centered evaluation paradigm} for XAI in bioacoustic disease detection, distinguishing model-centric faithfulness from task-level trustworthiness. Our formulation moves beyond single-model fidelity and reframes explanation evaluation as a statistically validated cross-model property without spatial ground truth.  \textbf{(2)} We propose \texttt{AGRI-Fidelity}, a unified metric that integrates \textbf{cross-model consensus} with fidelity-based causal validation. By stratifying agreement levels across architecturally diverse near-optimal models, the framework quantifies explanation stability within the consensus rather than relying on a single model instance.  \textbf{(3)} We develop a \textbf{tier-specific permutation-based null construction} using cyclic temporal shifts to estimate empirical False Discovery Rates (FDR). This design statistically suppresses continuous stationary artifacts while preserving sparse, time-localized physiological signals, without requiring retraining or external annotations.  \textbf{(4)} We provide \textbf{theoretical guarantees} demonstrating that stationary artifacts asymptotically yield FDR $\rightarrow 1$ (complete suppression), whereas sparse time-locked signals yield FDR $\rightarrow 0$, formally justifying the robustness of the proposed reliability formulation.  \textbf{(5)} We demonstrate through extensive experiments on real-world and controlled poultry vocalization datasets that \texttt{AGRI-Fidelity} consistently achieves \textbf{reliable differentiation between genuine bioacoustic markers and faithful yet domain-irrelevant artifacts}, outperforming standard masking-based evaluation metrics while exhibiting stable committee-level behavior across random initializations.

Collectively, these contributions establish a \textbf{statistically principled, explainer-agnostic framework} for evaluating listenable explanations in noisy real-world bioacoustic settings.

\section{Related Work} \label{sec:related-work}

This section reviews existing XAI evaluation frameworks, analyzes their methodological limitations in noisy bioacoustic settings, and highlights the gaps that motivate \texttt{AGRI-Fidelity}.

\subsection{XAI Evaluation Metrics and Their Limitations}
\noindent \textbf{Single-Model Fidelity Metrics.}
In XAI, fidelity measures the causal relationship between identified features and model predictions. Traditional fidelity metrics mask features deemed important by an XAI method and quantify the resulting confidence change. Core variants include Average Increase (AI) and Average Drop (AD)~\cite{chattopadhay2018grad}. AI computes the fraction of samples for which target-class confidence increases when only the masked-in region is retained, whereas AD measures the average relative confidence decrease when the input is masked. Similarly, Average Gain (AG)~\cite{zhang2024opti} evaluates the relative confidence increase after masking. To reduce local structure disruption, the Iterative Removal of Features (IROF) metric~\cite{rieger2020irof} groups inputs into meaningful segments and replaces flagged regions with a baseline value for efficient segment-level evaluation.

However, because these metrics assess causal alignment within a single model, they do not account for spurious correlations or cross-model instability, motivating a reliability-oriented evaluation framework such as \texttt{AGRI-Fidelity}.

\noindent \textbf{In-Distribution Fidelity Evaluation.}
A key limitation of standard masking is the generation of out-of-distribution (OOD) inputs, which can induce model shock and distort fidelity measurements. The Remove and Retrain (ROAR) framework~\cite{hooker2019benchmark} addresses this by removing important features and retraining the model to reassess feature importance, though at significant computational cost. To reduce this burden, the Remove and Debias (ROAD) framework~\cite{rong2022evaluating} avoids retraining and mitigates information leakage arising from mask structure. F-Fidelity~\cite{zheng2024f} similarly seeks to correct OOD-induced evaluation artifacts. Alternatively, Luss et al.~\cite{luss2021leveraging} generate contrastive explanations by projecting inputs into a latent space and identifying in-distribution modifications that flip predictions. 

While these approaches mitigate OOD bias, they remain focused on single-model behavior and do not address cross-model instability or statistically persistent spurious correlations.

\noindent \textbf{Formalizing XAI Evaluation Robustness.}
Beyond empirical masking, several works formalize explanation evaluation mathematically. Infidelity~\cite{yeh2019fidelity} grounds faithfulness in a Taylor expansion, requiring the explanation vector to approximate a local linear predictor of model behavior. Structural robustness metrics instead focus on explanation volatility. Sensitivity~\cite{yeh2019fidelity} uses adversarial optimization to measure the maximum divergence between explanation maps for nearly identical inputs, while Evaluation Stability~\cite{alvarez2018robustness} quantifies consistency via a local Lipschitz constant. Chalasani et al.~\cite{chalasani2020concise} further analyze faithfulness violations, showing that model smoothness through adversarial training is necessary for stable feature attribution.

However, these formulations primarily assess local smoothness and perturbation stability within a single model and do not address cross-model disagreement or statistically persistent spurious correlations in real-world noisy environments.

\noindent \textbf{Distance and Aggregation-Based Metrics.}  Some approaches improve explanation quality through aggregation or controlled benchmarking. Bhatt et al.~\cite{bhatt2020evaluating} aggregate multiple XAI methods for a single model to reduce complexity and produce smoother saliency maps. CLEVR-XAI~\cite{arras2022clevr} leverages synthetic datasets with known spatial ground-truth masks, enabling direct distance- or overlap-based evaluation.

However, aggregation does not resolve model-specific bias, and synthetic benchmarks with perfect masks are rarely available in real-world bioacoustic settings, limiting distance-based evaluation in noisy environments.

\subsection{Cross-Model Consensus and Its Limitations}

To address limitations of single-model evaluation, recent work explores model multiplicity and the Rashomon set~\cite{schwarzschild2023reckoning}, arguing that task-relevant features should generalize across near-optimal models. CAMO~\cite{yu2024camo} measures inter-model agreement via prediction and logic consensus using approximate Shapley values, while Li et al.~\cite{li2023cross} aggregate committee explanations and assess similarity to the aggregate baseline. Conformal prediction has also been used to construct uncertainty sets for explanations, providing distribution-free bounds on variability~\cite{marx2023but}. Cross-explainer approaches~\cite{dietz2024agree, thinn2026consensus} analyze agreement among multiple explanation methods applied to a single model.

Although these agreement-based frameworks capture overlap or bounded variability, they primarily quantify similarity rather than statistical legitimacy. Persistent but spurious features, such as stationary background artifacts, may exhibit high inter-model or inter-explainer agreement and thus appear robust. Moreover, similarity-based aggregation does not test whether observed consensus exceeds chance. These gaps motivate a statistically grounded consensus formulation with explicit null modeling, as implemented in \texttt{AGRI-Fidelity}.

\subsection{Limitations of Existing Evaluation and How \texttt{AGRI-Fidelity} Addresses Them}

Existing XAI evaluation frameworks exhibit three core limitations in real-world bioacoustic settings.

\noindent \textbf{(1) Spurious Faithfulness.}
Standard fidelity and in-distribution metrics ignore persistent spurious correlations~\cite{fresz2024classification,banerjee2022methods}. When models exploit stationary artifacts, masking-based metrics validate this reliance rather than flagging it as domain-irrelevant. ROAR~\cite{hooker2019benchmark} mitigates masking bias but requires costly retraining. In contrast, \texttt{AGRI-Fidelity} uses cross-model consensus with permutation-based False Discovery Rate (FDR) to suppress artifact-driven agreement without retraining.

\noindent \textbf{(2) Dependence on Ground-Truth Masks.}
Distance- and overlap-based metrics assume precise spatial or temporal annotations, which are rarely available in noisy farm audio. \texttt{AGRI-Fidelity} removes this requirement by constructing an empirical null via cyclic temporal permutation, enabling reliability assessment without ground-truth feature maps.

\noindent \textbf{(3) Agreement Without Statistical Validation.}
Consensus-based methods quantify inter-model similarity but do not test whether agreement exceeds chance alignment. Persistent artifacts may therefore appear robust due to high overlap alone. \texttt{AGRI-Fidelity} integrates fidelity with tier-specific FDR analysis, ensuring consensus reflects statistically validated causal structure rather than coincidental similarity.

\section{Problem Statement} \label{sec:prob-statement}
The primary challenge in evaluating feature attribution for bioacoustic disease detection is the presence of continuous stationary artifacts such as ventilation systems, that co-exist with the target signal. Formally, let $\mathcal{D} = \{(x_i, y_i)\}_{i=1}^N$ denote a dataset of poultry vocalization recordings $x_i$ with corresponding health labels $y_i \in \{\text{Healthy}, \text{Unhealthy}\}$ collected in real-world farm environments. Let $\mathcal{M} = \{M_1, \dots, M_K\}$ denote a set of trained classification models with comparable predictive performance on $\mathcal{D}$, and let $M^* \notin \mathcal{M}$ be a designated explainer model. Given an explanation method $\mathcal{E}$, we obtain an explanation map for an input $x$ by identifying time-frequency regions $P^*(x) = \mathcal{E}(M^*, x)$ deemed important for the prediction.

The goal of our evaluation metric is to \textbf{determine whether $P^*(x)$ provides a reliable interpretation of the model’s decision.} Existing XAI metrics primarily assess fidelity—whether masking $P^*(x)$ changes the prediction of $M^*$. However, \textit{fidelity alone does not ensure reliability} for two reasons. \textbf{(1) Model Multiplicity.} Near-optimal models $M_k \in \mathcal{M}$ may rely on different acoustic cues, yielding explanations $E(M_k, x)$ that vary substantially. An explanation faithful to $M^*$ may therefore capture model-specific behavior rather than shared task-relevant structure. \textbf{(2) Spurious Correlations in Noisy Environments.} Farm audio is noisy, enabling models to exploit persistent acoustic artifacts that are predictive but physiologically irrelevant. Explanations may faithfully highlight such features while reflecting \textbf{artifact-driven shortcuts} instead of genuine bioacoustic signals.

Therefore, a reliable explanation for audio-based poultry disease detection must satisfy three criteria: \textbf{(1) Fidelity.} It must identify time–frequency regions that causally influence the prediction. \textbf{(2) Stability.} It should be consistent across near-optimal models in the committee $\mathcal{M}$, reflecting task-relevant structure rather than continuous artifacts. \textbf{(3) Domain Alignment.} It must emphasize physiologically meaningful bioacoustic markers over background noise, enabling expert validation.

Accordingly, this work seeks to develop an explanation evaluation framework that jointly quantifies causal fidelity, cross-model stability, and domain alignment, enabling systematic assessment of explanation reliability in noisy, real-world bioacoustic disease detection settings.

\section{\texttt{AGRI-Fidelity}: Reliability-Oriented Evaluation Framework} \label{sec:agri-fidelity}

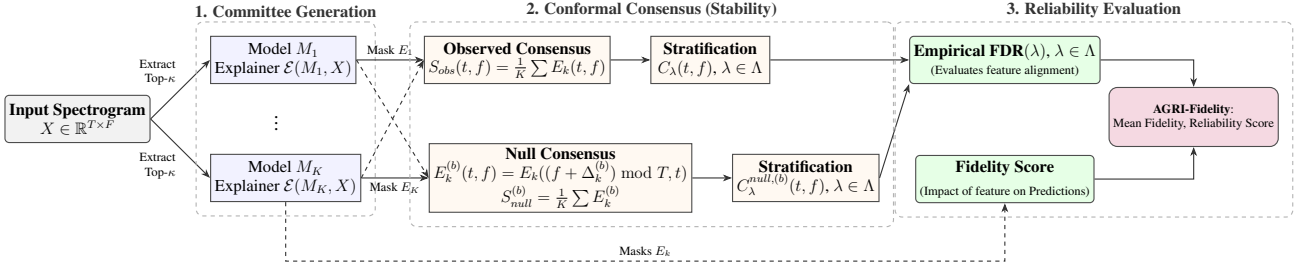
\begin{figure*}[!t]
\centering
\input{figures/proposed_approach_figure.tex}
\caption{\texttt{AGRI-Fidelity} Framework for Reliability-Oriented Evaluation via Tiered Cross-Model Consensus and Permutation-Based FDR Validation.}
\label{fig:flowdiagram}
\end{figure*}

This section presents the proposed reliability-oriented evaluation framework, detailing the construction of cross-model consensus, the permutation-based statistical validation via False Discovery Rate (FDR), and the integration of fidelity to quantify explanation reliability.

\subsection{Tiered Cross-Model Consensus}

Let $\mathcal{M}=\{M_1,\dots,M_K\}$ denote a committee of $K$ near-optimal classification models, where each $M_k$ is a trained classifier operating on spectrogram inputs (in our setting, $K=4$). Let $X \in \mathbb{R}^{T \times F}$ denote the log-magnitude spectrogram of an audio recording, where $T$ represents the number of time frames and $F$ the number of frequency bins.

Given $X$, an explainer $\mathcal{E}$ produces a continuous attribution map $\mathcal{E}(M_k, X) \in \mathbb{R}^{T \times F}$ for each model $M_k$. To isolate the most salient features, we retain only the top-$\tau$ highest-scoring time–frequency bins and binarize the result, yielding a binary explanation mask $E_k \in \{0,1\}^{T \times F}$ for each $M_k$, where $E_k(t,f)=1$ indicates that bin $(t,f)$ is deemed important for the prediction and $E_k(t,f)=0$ otherwise.
We compute the pixel-wise observed consensus score
\begin{equation}
S_{\text{obs}}(t,f) = \frac{1}{K}\sum_{k=1}^{K} E_k(t,f),
\end{equation}
which represents the proportion of models in the committee that identify time–frequency bin $(t,f)$ as salient. Since each $E_k(t,f)\in\{0,1\}$, $S_{\text{obs}}(t,f)$ takes discrete values in $\{0,\frac{1}{K},\dots,1\}$, where $S_{\text{obs}}(t,f)=1$ indicates full agreement and lower values reflect partial or model-specific attribution.
Because each $E_k$ is binary, $S_{\text{obs}}(t,f)$ assumes discrete agreement levels corresponding to the exact number of models selecting bin $(t,f)$. Let
\begin{equation}
\Lambda = \left\{ \frac{1}{K}, \frac{2}{K}, \dots, 1 \right\}
\end{equation}
denote the set of non-zero consensus tiers, where $\lambda = \frac{m}{K}$ indicates agreement among $m$ models. For $K=4$, $\Lambda=\{0.25,0.5,0.75,1.0\}$.
To retain this agreement structure, we stratify $S_{\text{obs}}$ into level-specific binary masks
\begin{equation}
C_{\lambda}(t,f) =
\begin{cases}
1, & S_{\text{obs}}(t,f) \ge \lambda, \\
0, & \text{otherwise},
\end{cases}
\end{equation}
for each $\lambda \in \Lambda$.
This tiered construction preserves the distinction between weak and strong inter-model agreement, preventing high-consensus regions (robust signals) from merging with low-consensus, model-specific attributions, thereby enabling reliability evaluation at different confidence thresholds.

\subsection{Permutation-Based Statistical Validation}

To determine whether cross-model consensus reflects meaningful structure rather than coincidental alignment, we construct an empirical null distribution using cyclic temporal permutation. Recall that $T$ denotes the number of time frames in the spectrogram. For each of $B$ permutation iterations, we independently shift each binary explanation mask $E_k \in \{0,1\}^{T\times F}$ along the time axis by a random offset $\Delta_k^{(b)} \sim \mathcal{U}\{0,\dots,T-1\}$:
\begin{equation}
E_k^{(b)}(t,f) = E_k((t+\Delta_k^{(b)}) \bmod T, f),
\end{equation}
where $b=1,\dots,B$. This operation preserves the marginal structure of each mask while destroying temporal alignment across models.
For each permutation $b$, we compute the corresponding null consensus
$S_{\text{null}}^{(b)}(t,f) = \frac{1}{K}\sum_{k=1}^{K} E_k^{(b)}(t,f),$
which represents agreement under random temporal alignment. Applying the same tiered stratification defined previously yields $C_{\lambda}^{\text{null},(b)}$ for each $\lambda \in \Lambda$.
We then compute the Empirical False Discovery Rate (FDR) for each consensus tier:
\begin{equation}
\text{FDR}(\lambda) =
\frac{\frac{1}{B}\sum_{b=1}^{B}\sum_{t,f} C_{\lambda}^{\text{null},(b)}(t,f) + 1}
{\sum_{t,f} C_{\lambda}(t,f) + 1}.
\end{equation}

Here, the numerator estimates the expected number of pixels reaching agreement level $\lambda$ under the null hypothesis of random temporal alignment, while the denominator counts the observed agreement in the original (unshifted) masks. Continuous stationary artifacts remain invariant under temporal shifts, yielding $\text{FDR}(\lambda)\approx 1$, whereas sparse, time-localized physiological signals lose alignment under permutation, producing $\text{FDR}(\lambda)\approx 0$. Hence, FDR quantifies whether cross-model consensus is statistically distinguishable from random alignment.

\subsection{Reliability Integration via Fidelity}

Statistical consensus alone does not guarantee causal influence on the model’s prediction. A region may exhibit low $\text{FDR}(\lambda=1)$ due to stable agreement yet remain irrelevant to the decision boundary. To verify importance, we also incorporate fidelity.

Formally, let $X \setminus C_{\lambda}$ denote the input after removing bins corresponding to consensus tier $\lambda$. Fidelity measures whether masking these regions degrades the predicted class confidence of model $M_k$. High fidelity indicates that the model actively depends on the selected regions for its prediction.
We define a feature as \textbf{\emph{reliable} only if it satisfies both low FDR (statistical legitimacy of agreement) and high fidelity (causal impact on prediction)}. This joint criterion ensures that cross-model consensus reflects validated task-relevant structure rather than artifact-driven or coincidental overlap. \textbf{Figure \ref{fig:flowdiagram}} shows our overall approach.

\section{Theoretical Guarantees of \texttt{AGRI-Fidelity}}\label{sec:theory}

In this section, we formally justify the Consensus-based metric. We prove that the cyclic-shift null model (i) provably suppresses strictly stationary artifacts by yielding maximal False Discovery Rate (FDR), and (ii) remains highly sensitive to sparse, time-localized bioacoustic signals.

\subsection{Consensus Formulation and Cyclic-Shift Null Model}

Let $X \in \mathbb{R}^{T \times F}$ denote the input spectrogram with $T$ temporal frames and $F$ frequency bins. For a committee of $K$ diverse models, let $B_k \in \{0, 1\}^{T \times F}$ represent the binary importance mask produced by the $k$-th model. The observed committee consensus at pixel $(t,f)$ is defined as
$
S_{\text{obs}}(t,f) = \frac{1}{K} \sum_{k=1}^K B_k(t,f).
$
To quantify accidental agreement among models, we construct a permutation-based null distribution using independent cyclic time shifts. Specifically, for each model $k$, we sample $\delta_k \sim U[0, T]$ and apply the shift along the temporal axis to obtain $B_k^{(b)}(t,f)$. The corresponding null consensus is
$
S_{\text{null}}^{(b)}(t,f) = \frac{1}{K} \sum_{k=1}^K B_k^{(b)}(t,f).
$

\subsection{Guarantee I: Safety Against Stationary Artifacts}

\begin{theorem}[\bf Safety against Stationary Artifacts]
\label{thm:safety}
If the committee models converge on a feature that is strictly stationary (time-invariant), the Consensus algorithm will asymptotically assign a False Discovery Rate (FDR) of $1.0$, resulting in a Reliability Score of $0$.
\end{theorem}

\begin{proof}
\textbf{Assumption of Stationarity:} 
Suppose the detected feature is strictly stationary, meaning the importance mask $E_k$ does not vary with time $t$. Formally, for all models~$k$, we assume $E_k(t,f) = \phi_k(f)$, where $\phi_k(f) \in \{0,1\}$ depends only on the frequency bin $f$ and is constant across all temporal indices.
\noindent \textbf{Oberved Count ($N_{obs}$):} 
Under this assumption, the consensus score at each frequency bin is identical across time. For any consensus threshold $\lambda \in [0, 1)$, the total number of observed positive pixels across the spectrogram is therefore:
\begin{equation}
\begin{aligned}
N_{obs}(\lambda)
&= \sum_{f=1}^{T} \sum_{t=1}^{F}
\mathbf{1}\left(
\frac{1}{K} \sum_{k=1}^K \phi_k(f) \ge \lambda
\right) \\
&= T \sum_{f=1}^{F}
\mathbf{1}\left(
\bar{\phi}(f) \ge \lambda
\right),
\end{aligned}
\end{equation}
where $\bar{\phi}(f) = \frac{1}{K} \sum_{k=1}^K \phi_k(f)$ denotes the average committee agreement at frequency $f$. The factor $T$ arises because each qualifying frequency bin contributes equally across time.

\noindent \textbf{Null Count ($N_{null}$):} 
Consider the temporally shifted mask $B_k(t, f+\Delta_k)$. Because the mask is time-invariant, we have $B_k(t, f+\Delta_k) = \phi_k(f)$ for all shifts $\Delta_k$. Thus, the cyclic permutation does not change the mask structure. Consequently, the shifted consensus equals the observed consensus:
\begin{equation}
S_{\Delta, obs}(t,f) = \frac{1}{K} \sum_{k=1}^K \phi_k(f) = S(t,f).
\end{equation}
It follows that the expected null count matches the observed count for every permutation sample:
\begin{equation}
\mathbb{E}[N_{null}(\lambda)] = N_{obs}(\lambda).
\end{equation}

\noindent \textbf{FDR and Reliability Calculation:} 
Applying the $+1$ smoothing term for numerical stability, the FDR becomes
\begin{equation}
\text{FDR}(\lambda) = \frac{\mathbb{E}[N_{null}(\lambda)] + 1}{N_{obs}(\lambda) + 1} = 1.
\end{equation}
Therefore, the Reliability Score evaluates to $R(t,f) = \max(0, 1 - \text{FDR}) = 0$, meaning the stationary artifact is completely suppressed by the consensus procedure.
\end{proof}

\subsection{Guarantee II: Sensitivity to Sparse Signals}

\begin{theorem}[\bf Sensitivity to Sparse Signals]
\label{thm:sensitivity}
If the committee models converge on a sparse, time-localized signal, the expected FDR approaches $\rho^{K-1}$ (where $\rho \ll 1$ is the sparsity ratio), allowing the feature to be retained with high confidence.
\end{theorem}

\begin{proof}
\textbf{Assumption of Sparsity and Alignment:} Let the true biological signal exist strictly within a brief time window of duration $\tau$, such that $\tau \ll T$. We define the sparsity ratio as $\rho = \tau/T$. Assume perfect committee alignment: all $K$ models correctly identify the feature, setting $E_k(t,f) = 1$ inside the window and $0$ elsewhere.

\noindent \textbf{Observed Count ($N_{obs}$):} Given perfect agreement inside the temporal window, the observed count for maximum consensus ($\lambda=1.0$) is:
\begin{equation}
N_{obs}(1.0) = F \cdot \tau.
\end{equation}

\noindent
\textbf{Null Count ($N_{null}$):} Under the null distribution, independent random shifts $(f+\Delta_k)$ are applied. For a pixel $(t,f)$ to achieve full consensus ($\lambda=1.0$) in the shifted state, all $K$ independent masks must stochastically align at that pixel. The marginal probability that model $k$'s shifted mask is active at time $t$ is exactly its sparsity ratio $\rho$. Assuming the shifts are independent and identically distributed, the joint probability of overlap is: $\mathbb{P}(S_\Delta(t,f) = 1.0) = \rho^K.$
The expected number of false discoveries over the entire spectrogram is therefore:
$\mathbb{E}[N_{null}(1.0)] = F \cdot T \cdot \rho^K.$

\noindent
\textbf{FDR Calculation:} Substituting the expected counts yields:
\begin{equation}
\text{FDR}(1.0) \approx \frac{\mathbb{E}[N_{null}(1.0)]}{N_{obs}(1.0)} = \frac{F \cdot T \cdot \rho^K}{F \cdot \tau} = \frac{T \cdot \rho^K}{\tau}.
\end{equation}
Recognizing that $\tau/T = \rho$, this simplifies to:
\begin{equation}
\text{FDR}(1.0) \approx \frac{\rho^K}{\rho} = \rho^{K-1}.
\end{equation}

\noindent Because the signal is sparse ($\rho \ll 1$) and $K \ge 2$, the term $\rho^{K-1}$ becomes negligible. Consequently, the Reliability Score converges to $R \approx 1 - \rho^{K-1} \approx 1$, preserving the time-locked biological marker.
\end{proof}

\section{Experimental Design \& Setup}

This section describes the datasets, baseline metrics and explainers, and the time–frequency representation and preprocessing configuration used in all \texttt{AGRI-Fidelity} evaluations.
\subsection{Datasets}
\begin{itemize}
    \item \textbf{Poultry Dataset} \cite{aworinde2023poultry}: Audio was collected at 96 kHz with 24-bit resolution over 65 days. For our experiments, we resampled recordings to 22 kHz or 24 kHz as required.
    
    \item \textbf{Poultry Dataset with a Spurious Feature}: A controlled variant constructed by injecting high-frequency noise into all unhealthy samples of the original Poultry Dataset \cite{aworinde2023poultry}. 

    \item \textbf{SmartEars Poultry} \cite{qiao2026smartears}: A 6,000-sample subset of the SmartEars dataset containing five-second clips labeled Healthy, Unhealthy, or None. We retained only the Healthy and Unhealthy samples for analysis.
    
\item \textbf{SmartEars Poultry Denoised}: The same dataset \cite{qiao2026smartears}, processed using Python's \texttt{noisereduce} and Harmonic-Percussive Source Separation (HPSS) for noise reduction.
    
\item \textbf{SwineCough} \cite{smartfarmkorea_swinecough}: Contains dry and abdominal swine cough audio. This clean dataset of cough-only sounds was included to evaluate metric behavior under minimal background noise.
\end{itemize}

\begin{figure*}[t]
\centering

\begin{subfigure}{0.23\linewidth}
    \centering
    \includegraphics[width=\linewidth]{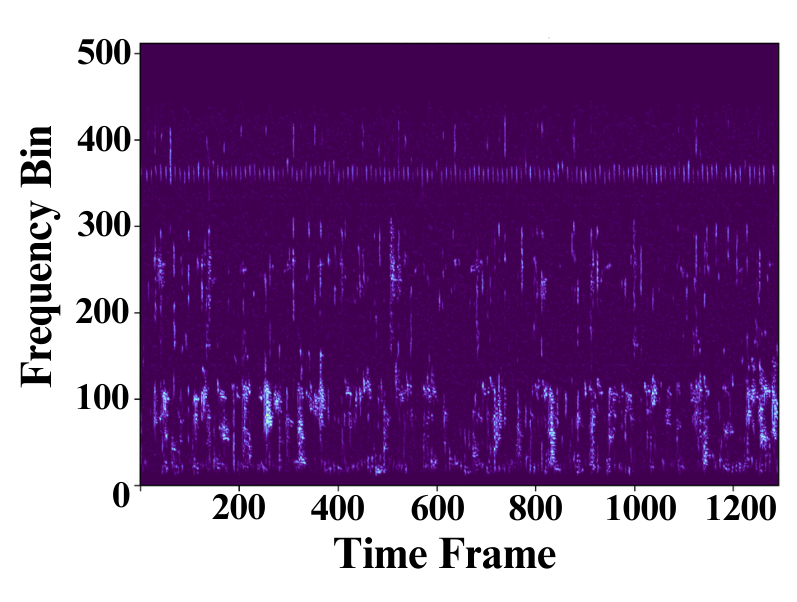}
    \caption{CNN}
\end{subfigure}
\hfill
\begin{subfigure}{0.23\linewidth}
    \centering
    \includegraphics[width=\linewidth]{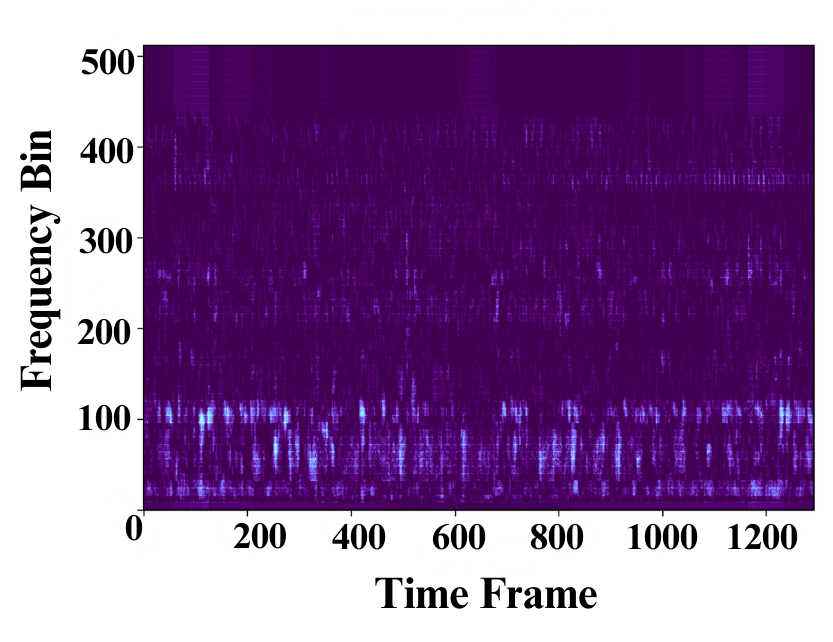}
    \caption{MLP}
\end{subfigure}
\hfill
\begin{subfigure}{0.23\linewidth}
    \centering
    \includegraphics[width=\linewidth]{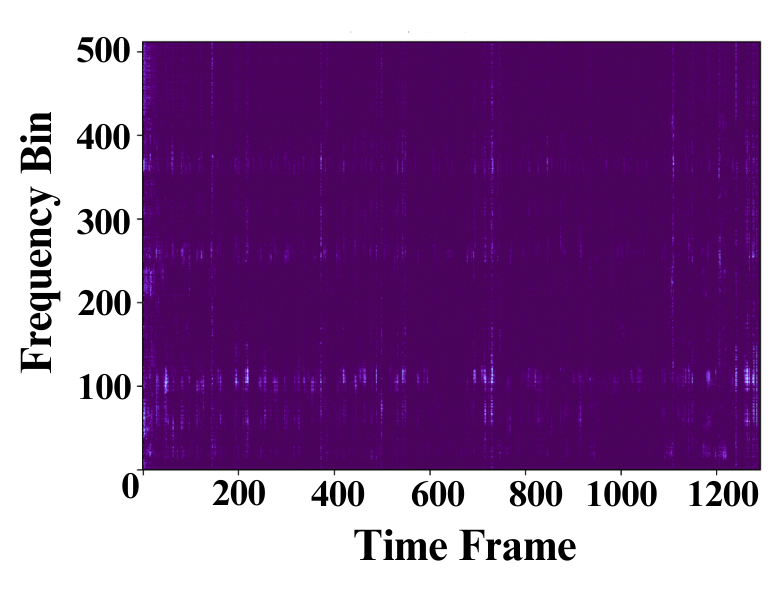}
    \caption{LSTM}
\end{subfigure}
\hfill
\begin{subfigure}{0.23\linewidth}
    \centering
    \includegraphics[width=\linewidth]{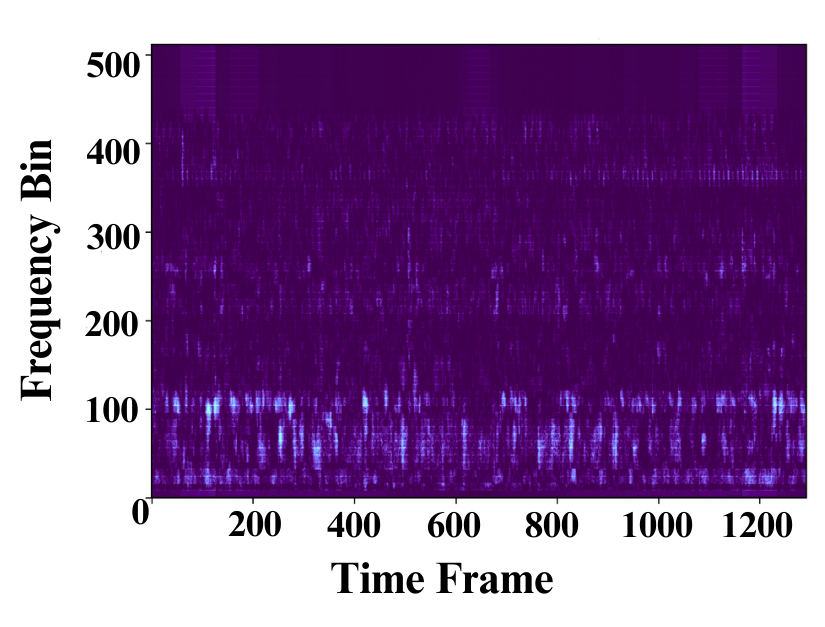}
    \caption{ResNet}
\end{subfigure}

\caption{\textbf{Cross-Model Attribution Consistency:} Integrated gradients for a healthy sample in the denoised poultry dataset across CNN, MLP, LSTM, and ResNet consensus models.}
\label{fig:exampleIGs}
\end{figure*}



\subsection{Baseline Metrics}

We compare \texttt{AGRI-Fidelity} with the following metrics.
\noindent\textbf{(1) Faithfulness~\cite{parekh2022listen}.}
Measures whether explanation-selected features alone preserve the model’s prediction given by:
\begin{equation}
\mathrm{F}_n = p_c(X_n^{\text{int}}).  
\end{equation}

Here, $X_n^{\text{int}}$ retains only explanation-relevant features for sample $n$, and $p_c(\cdot)$ is the predicted probability for class $c$. Larger values indicate stronger preservation.

\noindent\textbf{(2) Fidelity~\cite{tomsett2020sanity}.}
Quantifies the confidence drop after removing the explanation region:
\begin{equation}
\mathrm{FF}_n
=
p_c(X_n)
-
p_c\left(X_n \odot (1 - M_n)\right).
\label{eq:faithfulness}
\end{equation}
$M_n$ denotes the explanation mask; larger positive values indicate greater model reliance on the highlighted region.

\noindent\textbf{(3) Average Increase (AI)~\cite{chattopadhay2018grad}.}
Percentage of samples where confidence increases when explanation region is retained:
\begin{equation}
\mathrm{AI}
=
\frac{1}{N}\sum_{n=1}^{N}
\mathbb{I}\left[p_c(X_n \odot M) > p_c(X_n)\right]
\times 100.
\label{eq:ai}
\end{equation}

\noindent\textbf{(4) Average Drop (AD)~\cite{chattopadhay2018grad}.}
Relative confidence decrease under masking with smaller values being better.
\begin{equation}
\mathrm{AD}
=
\frac{1}{N}\sum_{n=1}^{N}
\frac{\max(0,\, p_c(X_n) - p_c(X_n \odot M))}{p_c(X_n)}
\times 100.
\label{eq:ad}
\end{equation}

\noindent\textbf{(5) Average Gain (AG)~\cite{zhang2024opti}.}
Relative confidence increase after masking with larger values indicating stronger confidence gain.
\begin{equation}
\mathrm{AG}
=
\frac{1}{N}\sum_{n=1}^{N}
\frac{\max(0,\, p_c(X_n \odot M) - p_c(X_n))}{1 - p_c(X_n)}
\times 100.
\label{eq:ag}
\end{equation}

\noindent\textbf{(6) Sparseness~\cite{chalasani2020concise}.}
Measures if predictions depend on a small set of salient features; larger values indicate better explanations.

\noindent\textbf{(7) Complexity~\cite{bhatt2020evaluating}.}
Entropy of attribution scores; smaller values indicate tighter localization.
Anonymized code available here: \url{https://anonymous.4open.science/r/AGRI-Fidelity-AFF3/README.md}.




\subsection{Baseline Explainers} \label{subsec:baseline-explainers}

To evaluate \texttt{AGRI-Fidelity}, we consider four state-of-the-art listenable explanation methods. \textbf{(1) Listen to Interpret (L2I)~\cite{parekh2022listen}.} L2I generates sparse, class-discriminative time–frequency masks using non-negative matrix factorization to create explanations. \textbf{(2) Listenable Maps for Audio Classifier in Time Domain (LMAC-TD)~\cite{mancini2025lmac}.} LMAC-TD is a post-hoc method that trains a decoder to produce explanations directly in the time domain. Both L2I and LMAC-TD were implemented using the SpeechBrain Toolkit~\cite{speechbrain}. \textbf{(3) CoughLIME~\cite{wullenweber2022coughlime}.} CoughLIME adapts LIME to audio by perturbing spectrogram regions and fitting a local surrogate model to explain cough-based predictions. \textbf{(4) AudioLIME~\cite{wullenweber2022coughlime}.} AudioLIME extends this LIME-based approach to general audio, producing localized explanations via surrogate modeling.

\begin{table*}[t]
\centering
\caption{\textbf{Quantitative Comparison of Baseline Explainers Across Standard Metrics and \texttt{AGRI-Fidelity} Components:} Average Increase (AI$\uparrow$), Average Drop (AD$\downarrow$), Average Gain (AG$\uparrow$), Sparseness$\uparrow$, Complexity$\downarrow$, Faithfulness$\uparrow$, and \texttt{AGRI-Fidelity} (Fidelity$\uparrow$ and Reliability$\uparrow$) for the Poultry dataset with a spurious feature.}
\label{tab:baseline_comparison}
\scriptsize
\setlength{\tabcolsep}{4pt}
\resizebox{\textwidth}{!}{
\begin{tabular}{@{}lcccccccc@{}}
\toprule
Baseline 
& AI$\uparrow$ 
& AD$\downarrow$ 
& AG$\uparrow$ 
& Sparse$\uparrow$ 
& Comp.$\downarrow$ 
& Faith.$\uparrow$ 
& \multicolumn{2}{c}{\texttt{AGRI-Fidelity}} \\
\cmidrule(lr){8-9}
 &  &  &  &  &  &  & Fidelity$\uparrow$ & Reliability$\uparrow$ \\
\midrule

LMAC-TD     
& 34.0500 $\pm$ 3.1955 
& 5.3698 $\pm$ 1.6220 
& 21.3799 $\pm$ 0.8626 
& 0.0140 $\pm$ 0.0033 
& 12.2920 $\pm$ 0.0001 
& 0.0739 $\pm$ 0.0243 
& 0.9128 $\pm$ 0.0273 
& 0.4285 $\pm$ 0.3958 \\

L2I       
& 30.1667 $\pm$ 3.9847  
& 7.7010 $\pm$ 1.2752 
& 20.3380 $\pm$ 2.9614 
& 0.0661 $\pm$ 0.0086 
& 12.2993 $\pm$ 0.0019 
& 0.1471 $\pm$ 0.0300 
& 0.5600 $\pm$ 0.0343 
& 0.3503 $\pm$ 0.3832 \\

CoughLIME 
& \textbf{49.8724 $\pm$ 40.7621} 
& 44.5938 $\pm$ 43.4127 
& \textbf{61.0600 $\pm$ 47.8459} 
& 0.3315 $\pm$ 0.2987 
& 13.3498 $\pm$ 0.0031 
& 0.1214 $\pm$ 0.1047 
& 0.1501 $\pm$ 0.0601
& \textbf{0.6715 $\pm$ 0.3289} \\

AudioLIME 
& 27.6923 $\pm$ 3.7621 
& \textbf{3.2820 $\pm$ 3.4127} 
& 2.6643 $\pm$ 2.8459 
& 0.4011 $\pm$ 0.0187 
& 0.2091 $\pm$ 0.0031 
& 0.3619 $\pm$ 0.0368 
& \textbf{0.9157 $\pm$ 0.0214}
& 0.3987 $\pm$ 0.3849 \\

\bottomrule
\end{tabular}
}
\end{table*}

\section{Experimental Results \& Analyses} \label{sec:exp-results-analyses}

This section presents baseline explainer comparisons, consensus model performance across datasets, and analyses of FDR behavior and stability. 
We begin with cross-method evaluation to benchmark explanation quality, then examine architectural robustness and committee-level reliability to characterize stability beyond masking-based metrics.

\subsection{Comparison with Listenable Baseline Explainers}

\noindent \textbf{Table~\ref{tab:baseline_comparison}} compares LMAC-TD, L2I, CoughLIME, and AudioLIME under standard masking metrics and \texttt{AGRI-Fidelity}.

\noindent
\textbf{Interpretability with Reliability Scores.}
Under traditional masking metrics, CoughLIME appears to outperform most baselines, achieving the highest Average Increase (49.87) and Average Gain (61.06). AudioLIME attains the lowest Average Drop (3.28), which is preferred under this metric. Based solely on these masking-based measures, one might conclude that CoughLIME (or partially AudioLIME) provides the most reliable explanations. However, masking metrics evaluate confidence shifts within a single model and do not assess statistical legitimacy across models. Although CoughLIME exhibits large confidence changes, its reliability score (0.6715) reflects only moderate cross-model stability, while other explainers exhibit even lower reliability despite competitive masking performance. This discrepancy indicates that masking-based metrics can reward explanations that exploit redundant or domain-irrelevant artifacts. In contrast, \texttt{AGRI-Fidelity} explicitly evaluates whether highlighted regions are statistically stable and structurally consistent across the model committee. High impact on model confidence does not necessarily imply trustworthiness, and \texttt{AGRI-Fidelity} provides a mechanism for reliability-aware interpretability assessment.

\noindent
\textbf{Decoupling Causal Influence from Statistical Reliability.}
The AGRI-Fidelity fidelity component reveals further separation between causal influence and reliability. AudioLIME achieves the highest fidelity score (0.9157), closely followed by LMAC-TD (0.9128), indicating a strong causal impact on model predictions. In contrast, CoughLIME’s fidelity score (0.1501) is substantially lower despite its strong masking metrics. Meanwhile, Reliability scores vary considerably across methods, with CoughLIME achieving the highest value (0.6715) and others exhibiting lower stability. This divergence demonstrates that causal influence and statistical reliability are orthogonal properties. By introducing Reliability as a separate evaluation axis, \texttt{AGRI-Fidelity} disentangles the magnitude of prediction impact from cross-model structural consistency, yielding a more principled assessment of explanation quality.

\noindent
\textbf{Structurally Inconsistent Attributions Analysis.}
\texttt{AGRI-Fidelity} is sensitive not only to sparsity but to structural alignment across the model committee. Although L2I produces explanations with moderate sparseness (0.0661) and complexity comparable to LMAC-TD, it achieves relatively lower AGRI-Fidelity components, with Fidelity (0.5600) substantially below LMAC-TD (0.9128) and AudioLIME (0.9157), and markedly lower Reliability (0.3503). This demonstrates that \texttt{AGRI-Fidelity} does not reward sparse or compact representations indiscriminately. Even when an explainer produces localized masks, fragmented or model-specific attributions that fail to align across the committee are penalized by the cyclic permutation null model. As a result, \texttt{AGRI-Fidelity} filters out attributions that lack structural agreement, ensuring that higher scores correspond to statistically validated and committee-consistent explanatory patterns rather than isolated saliency artifacts.

\noindent
\textbf{Understanding the High Standard Deviations of LIME-Based Explainers.}
The relatively large standard deviations observed for LIME-based explainers in \textbf{Table~\ref{tab:baseline_comparison}} primarily arise from the properties of the metric formulations and the aggregation strategy. First, Average Increase (AI) is a binary per-sample metric. When the success probability approaches $50\%$, the theoretical standard deviation $100\sqrt{p(1-p)}$ approaches its maximum ($\approx 50$), explaining values such as $49.9 \pm 40.8$ for CoughLIME. Second, Average Drop (AD) and Average Gain (AG) are ratio-based metrics whose denominators depend on the original prediction confidence ($p_{\text{orig}}$ and $1-p_{\text{orig}}$, respectively). When these probabilities approach $0$ or $1$, the ratios become numerically unstable, and a small number of samples can generate large outliers, leading to heavy-tailed distributions. Additionally, LIME-based explainers rely on stochastic perturbation sampling and local surrogate fitting, introducing inherent randomness absent in trained explanation models such as LMAC-TD or L2I.  Because the reported statistics are computed as standard deviations across individual samples rather than random seeds, the dispersion reflects genuine per-sample variability, which also propagates to the Reliability estimates when consensus stability differs substantially across inputs.

\begin{table}[t]
\centering
\caption{\textbf{Performance Comparison of Consensus Architectures Across Datasets:} For each dataset, we report the performance of CNN, MLP, LSTM, and ResNet consensus models. Columns report Accuracy (Acc$\uparrow$), Precision (Prec$\uparrow$), Recall (Rec$\uparrow$), F1$\uparrow$, Specificity (Spec$\uparrow$), and False Positive Rate (FPR$\downarrow$).}
\label{tab:all_models}

\scriptsize
\setlength{\tabcolsep}{6pt}

\resizebox{\columnwidth}{!}{
\begin{tabular}{p{1cm}lcccccc@{}}
\toprule
Dataset & Model & Acc$\uparrow$ & Prec$\uparrow$ & Rec$\uparrow$ & F1$\uparrow$ & Spec$\uparrow$ & FPR$\downarrow$ \\
\midrule
\multirow{4}{*}{\shortstack{\textbf{Chicken}\\\textbf{Denoised}}}
 & CNN     & 0.8654 & 0.9306 & 0.7667 & 0.8231 & 0.9360 & 0.0640 \\
 & MLP     & 0.9096 & 0.9189 & 0.8781 & 0.8954 & 0.9322 & 0.0678 \\
 & LSTM    & \textbf{0.9615} & \textbf{0.9987} & 0.8889 & \textbf{0.9412} & \textbf{0.9998} & \textbf{0.0100} \\
 & ResNet  & 0.9385 & 0.9321 & \textbf{0.9436} & 0.9366 & 0.9313 & 0.0687 \\
\midrule
\multirow{4}{*}{\shortstack{\textbf{Chicken}\\\textbf{Spurious}}}
 & CNN     & \textbf{0.9865} & 0.9805 & \textbf{0.9913} & \textbf{0.9855} & 0.9823 & 0.0177 \\
 & MLP     & 0.9365 & 0.9526 & 0.9036 & 0.9229 & 0.9618 & 0.0382 \\
 & LSTM    & 0.9423 & 0.9412 & 0.8889 & 0.9143 & 0.9706 & 0.0294 \\
 & ResNet  & 0.9500 & \textbf{0.9960} & 0.8962 & 0.9338 & \textbf{0.9964} & \textbf{0.0036} \\
\midrule
\multirow{4}{*}{\shortstack{\textbf{SmartEars}\\\textbf{Clean}}}
 & CNN     & 0.7399 & 0.7170 & 0.7886 & 0.7498 & 0.6901 & 0.3099 \\
 & MLP     & 0.6710 & 0.6741 & 0.6661 & 0.6651 & 0.6756 & 0.3244 \\
 & LSTM    & \textbf{0.8413} & \textbf{0.8797} & 0.7575 & \textbf{0.8141} & \textbf{0.9122} & \textbf{0.0878} \\
 & ResNet  & 0.8065 & 0.7940 & \textbf{0.8299} & 0.8098 & 0.7827 & 0.2173 \\
\midrule
\multirow{4}{*}{\shortstack{\textbf{SmartEars}\\\textbf{Noisy}}}
 & CNN     & 0.8311 & 0.8506 & 0.9236 & 0.8853 & 0.6068 & 0.3932 \\
 & MLP     & \textbf{0.8967} & 0.8973 & 0.9788 & \textbf{0.9363} & 0.6135 & 0.3865 \\
 & LSTM    & 0.8235 & 0.8250 & \textbf{0.9848} & 0.8976 & 0.2391 & 0.7609 \\
 & ResNet  & 0.8816 & \textbf{0.9072} & 0.9460 & 0.9256 & \textbf{0.6532} & \textbf{0.3468} \\
\midrule
\multirow{4}{*}{\textbf{SwineCough}} 
 & CNN     & 0.9613 & 0.9553 & 0.9575 & 0.9558 & 0.9637 & 0.0363 \\
 & MLP     & 0.9694 & 0.9648 & 0.9654 & 0.9650 & 0.9724 & 0.0276 \\
 & LSTM    & 0.9593 & 0.9516 & 0.9555 & 0.9535 & 0.9621 & 0.0379 \\
 & ResNet  & \textbf{0.9817} & \textbf{0.9819} & \textbf{0.9761} & \textbf{0.9790} & \textbf{0.9860} & \textbf{0.0140} \\
\bottomrule
\end{tabular}
}
\end{table}

\begin{figure*}[t]
\centering

\begin{subfigure}{0.19\textwidth}
    \includegraphics[width=\linewidth]{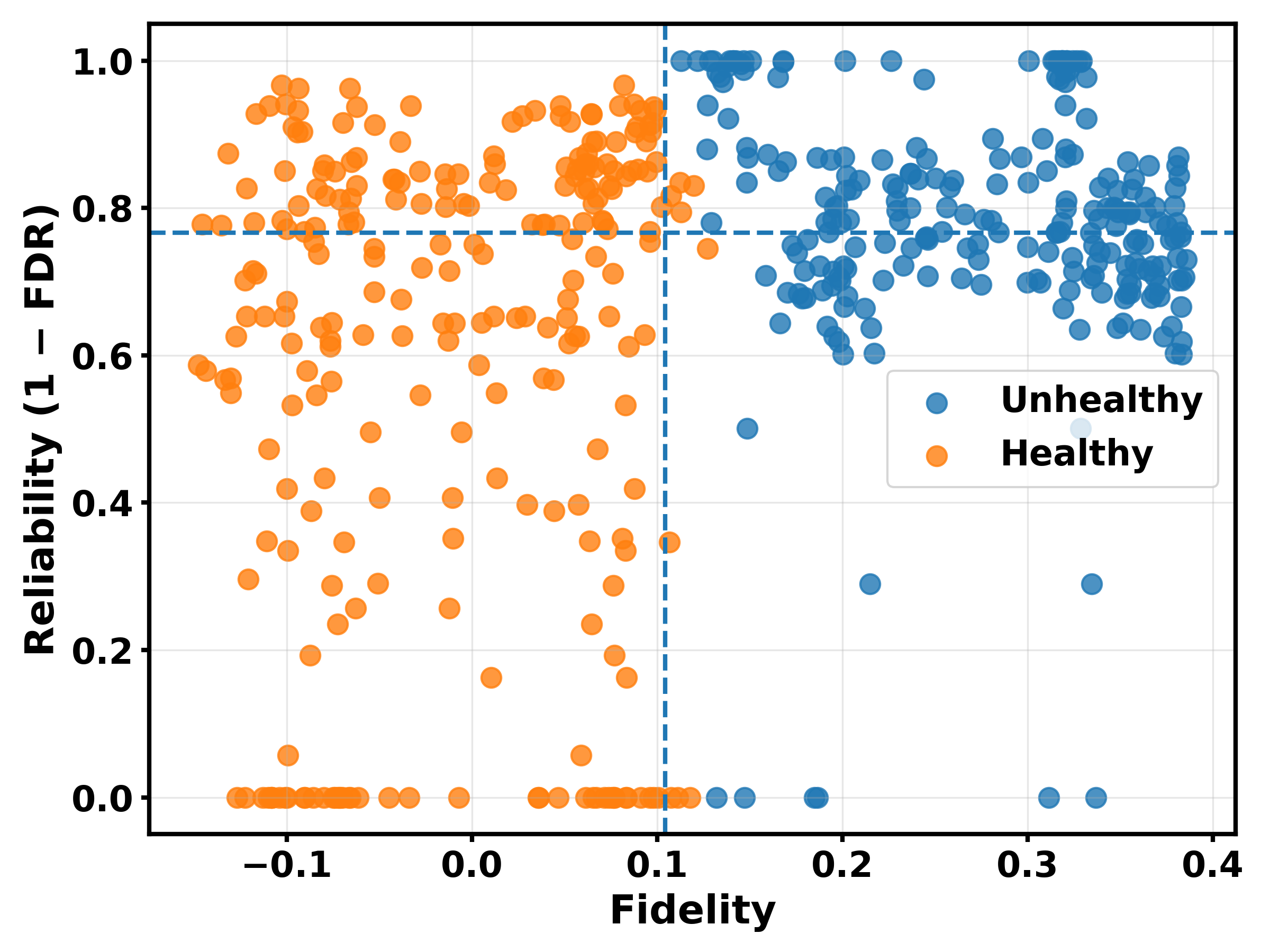}
    \caption{Chicken (Denoised)}
\end{subfigure}
\hfill
\begin{subfigure}{0.19\textwidth}
    \includegraphics[width=\linewidth]{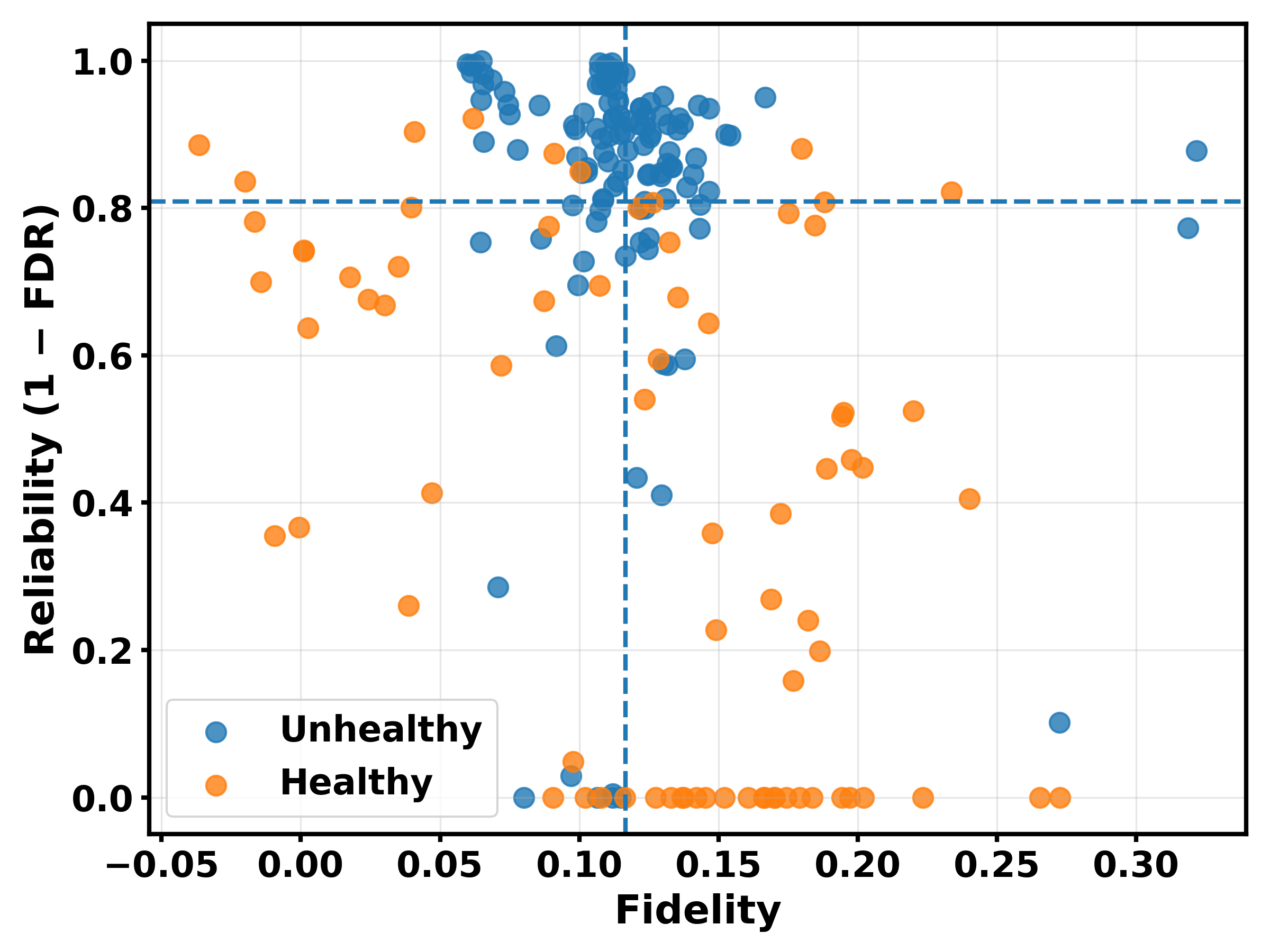}
    \caption{Chicken (Spurious)}
\end{subfigure}
\hfill
\begin{subfigure}{0.19\textwidth}
    \includegraphics[width=\linewidth]{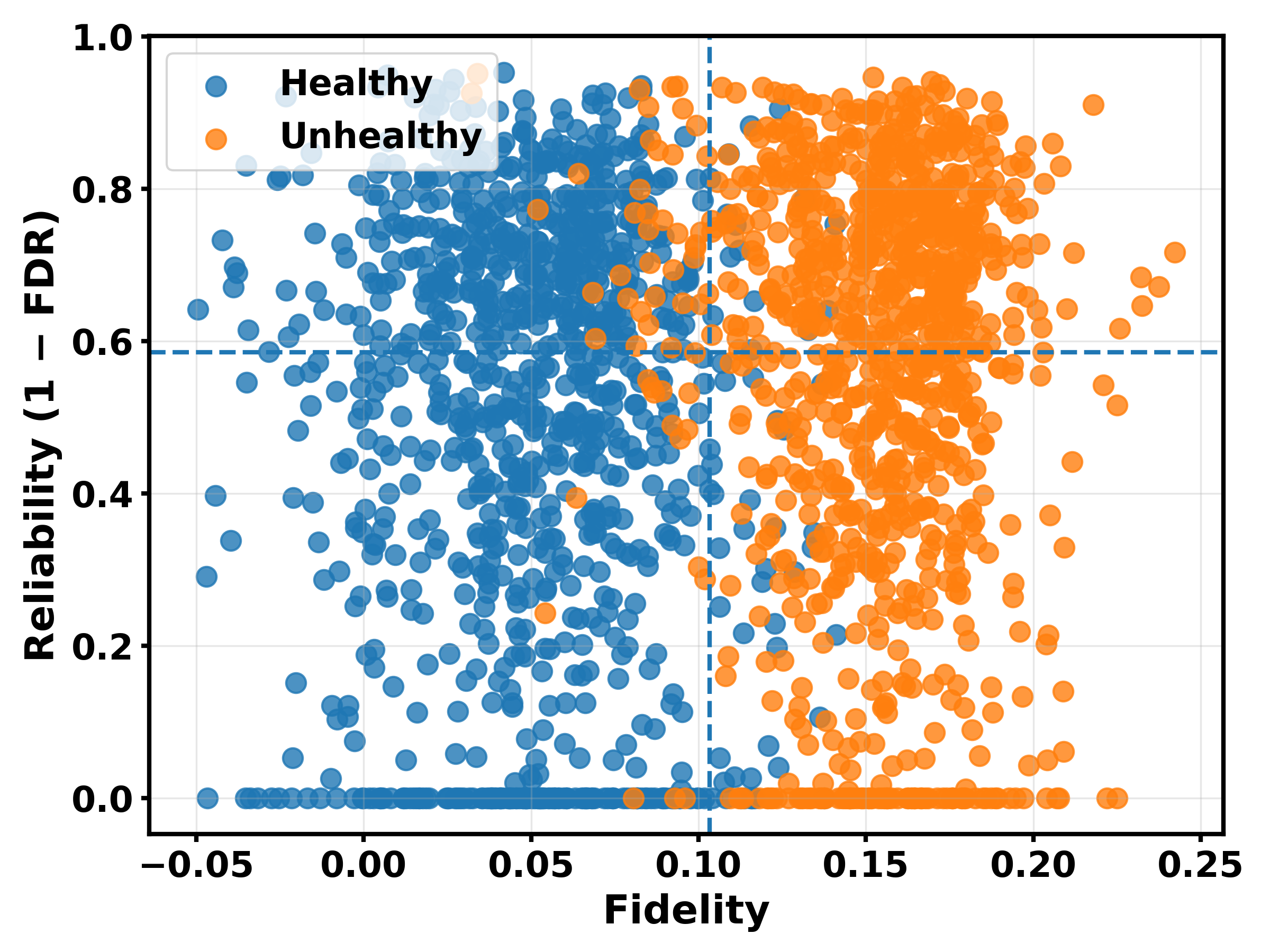}
    \caption{SmartEars (Denoised)}
\end{subfigure}
\hfill
\begin{subfigure}{0.19\textwidth}
    \includegraphics[width=\linewidth]{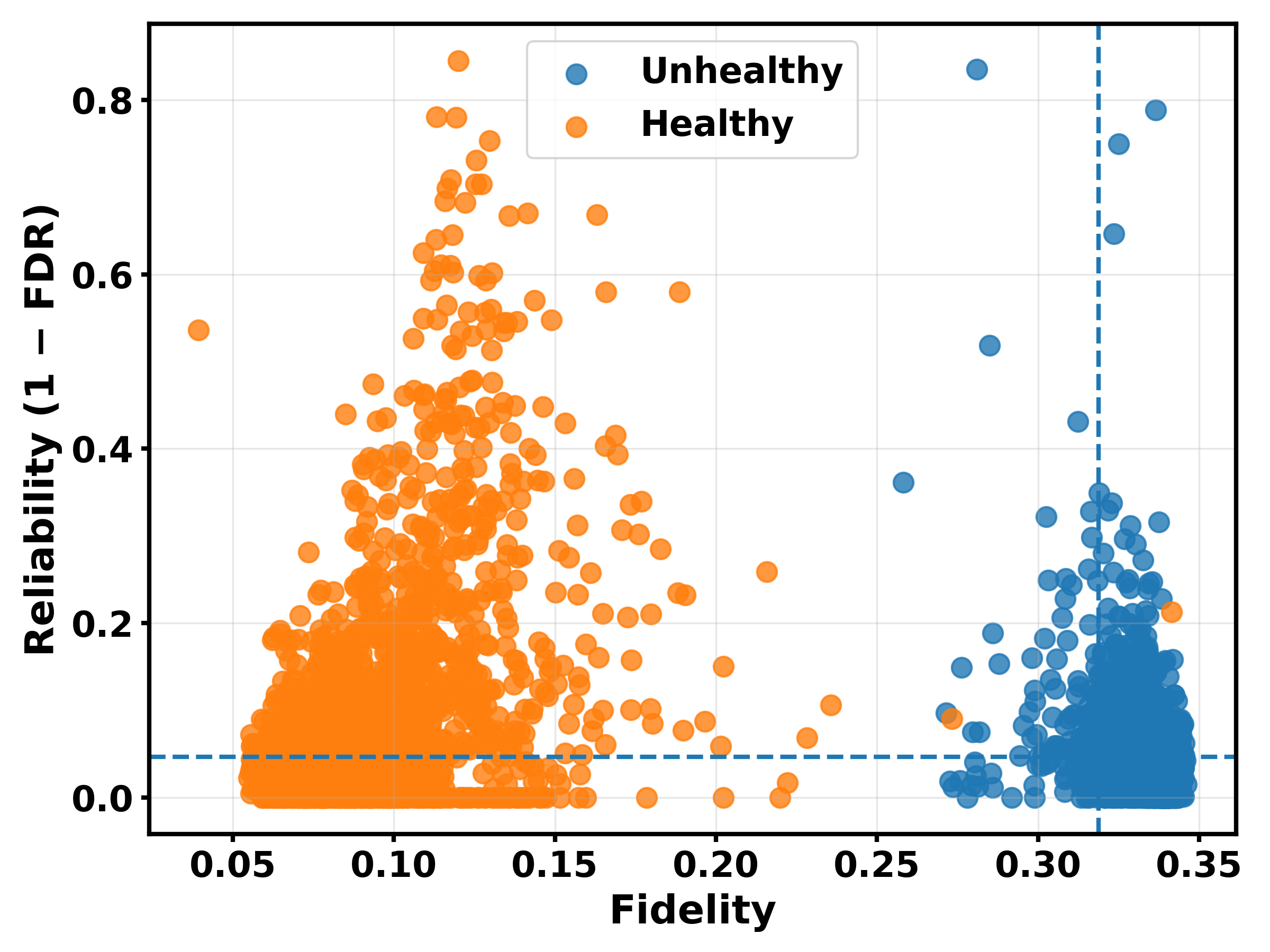}
    \caption{SmartEars (Raw)}
\end{subfigure}
\hfill
\begin{subfigure}{0.19\textwidth}
    \includegraphics[width=\linewidth]{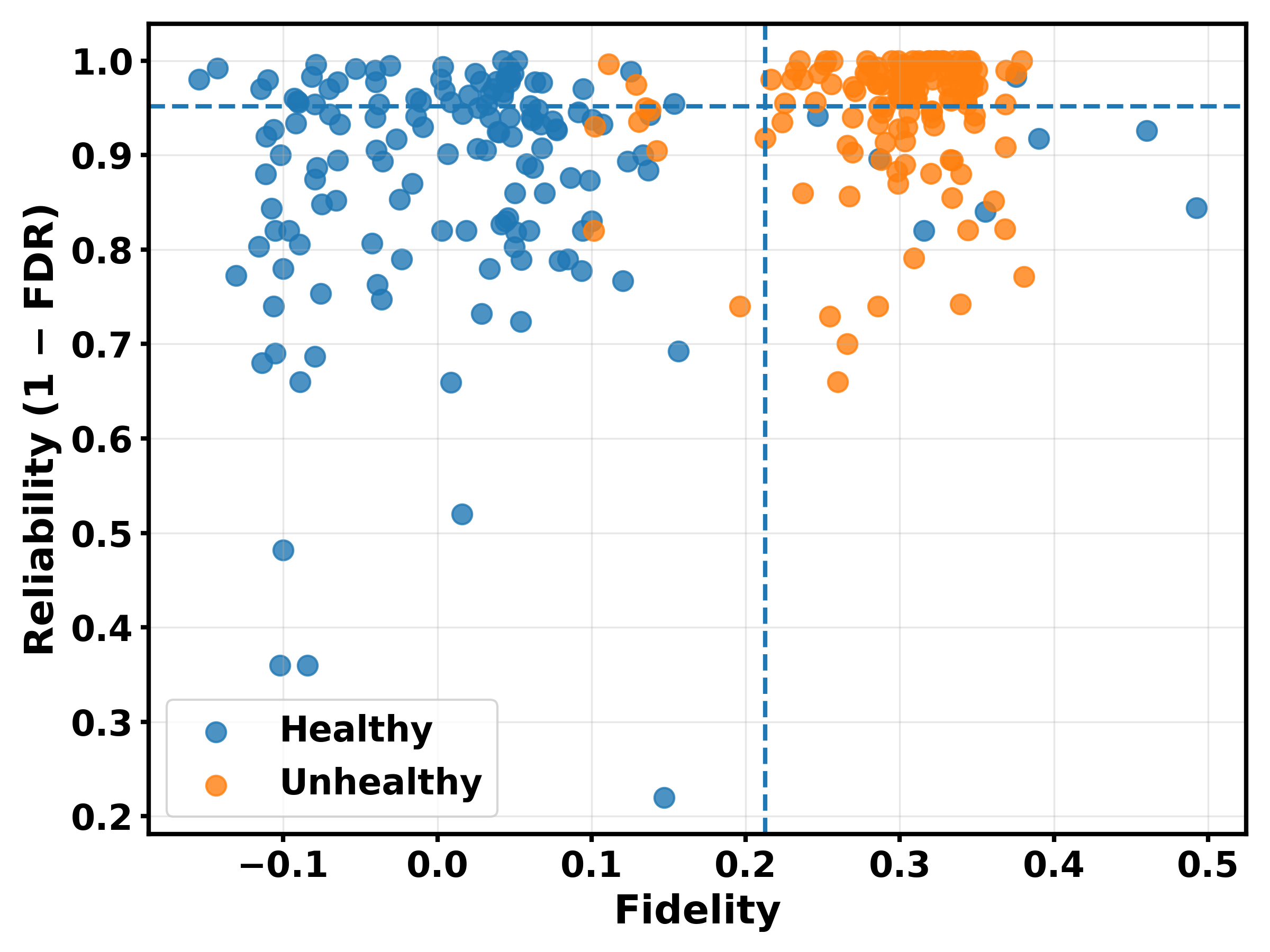}
    \caption{SwineCough}
\end{subfigure}

\caption{\textbf{Fidelity–Reliability Quadrant Distributions Across Datasets.}
Each plot shows per-sample \texttt{AGRI-Fidelity} components using CoughLIME Explainer. Denoised datasets exhibit stronger concentration in the high-fidelity, high-reliability quadrant, whereas noisy datasets demonstrate reliability dispersion under spurious contamination.}
\label{fig:quadrant_all_datasets}

\end{figure*}

\begin{figure}[t]
    \centering
    \includegraphics[width=0.6\linewidth]{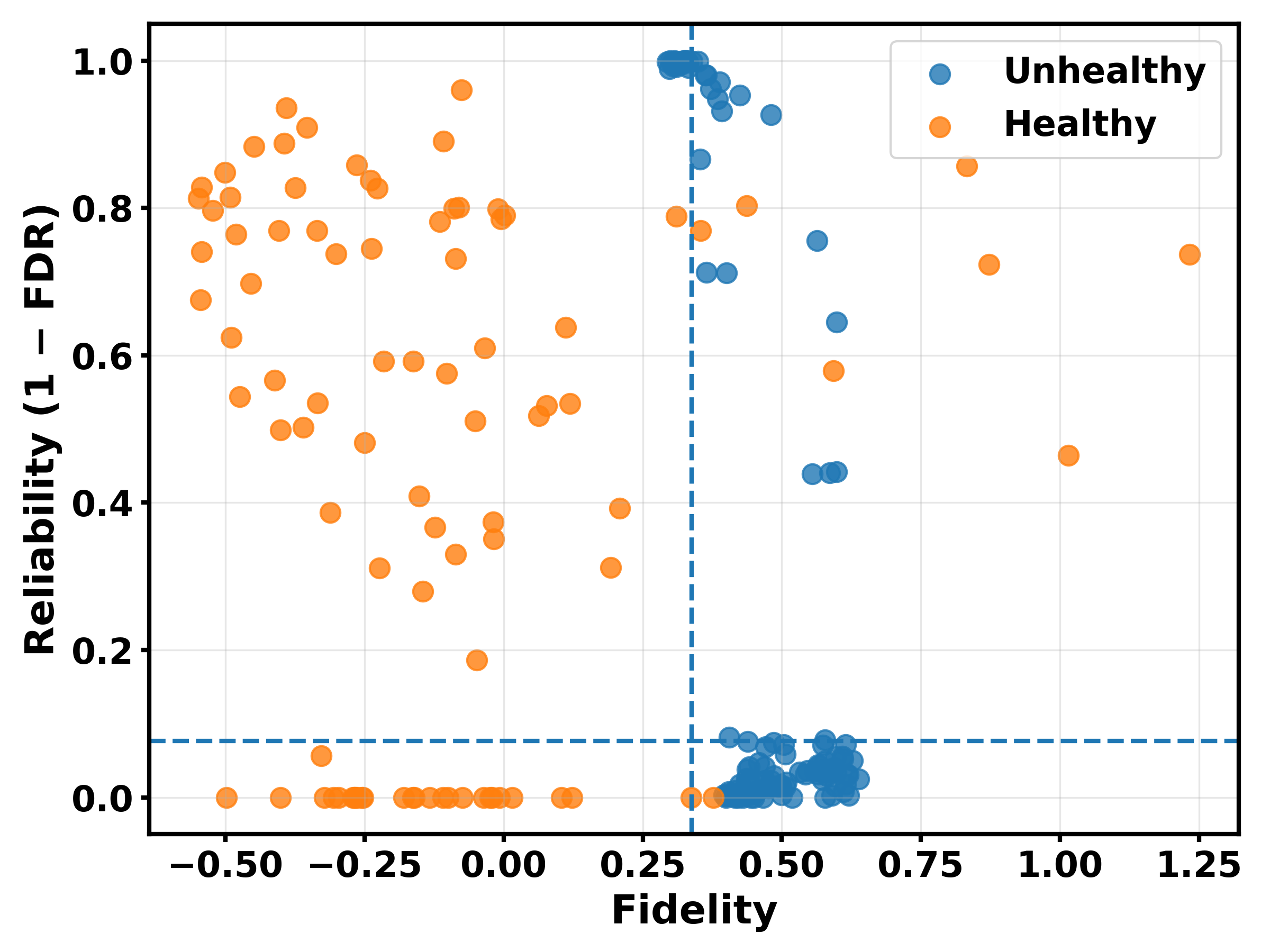}
    \caption{Fidelity–Reliability Quadrant Analysis on the Controlled Spurious Dataset. Each point represents a sample plotted by its mean fidelity and reliability score. Dashed lines denote median splits, forming four interpretative quadrants.}
    \label{fig:quadrant}
\end{figure}

\subsection{Fidelity–Reliability Distributions Across Datasets}

\textbf{Limitations of One-Dimensional Fidelity Evaluation.}
Comparing the Chicken Denoised and Chicken Spurious datasets reveals a limitation of fidelity-only evaluation. Along the fidelity axis, spurious artifacts in the Unhealthy class of the Spurious dataset achieve scores nearly indistinguishable from genuine signals in the Denoised dataset. Introducing Reliability as an orthogonal dimension, \texttt{AGRI-Fidelity} resolves this ambiguity by shifting artifact-driven explanations from the high-quality quadrant to the low-reliability regime.

\noindent
\textbf{Class-Specific Decoupling Under Spurious Injection.}
In \textbf{Figure~\ref{fig:quadrant_all_datasets}(b)}, unhealthy samples exhibit relatively low fidelity but high reliability. The injected spurious cue is continuous and stationary. Because CoughLIME operates locally through perturbations, it evaluates predictions within limited temporal neighborhoods rather than across the full background. As a result, several Unhealthy recordings emphasize sparse acoustic events instead of the globally present artifact. Near-zero FDR values dominate the consensus, yielding high reliability. \textbf{Figure~\ref{fig:quadrant}} further shows reliability computed using only consensus models, where fidelity is higher for Unhealthy samples but reliability remains low for most cases, consistent with \textbf{Table~\ref{tab:fdr_trend_distribution}}.

\noindent
\textbf{Diagnosing Shortcut Learning in Real-World Farm Audio.}
In SmartEars (Denoised), explanations span both high- and low-reliability regions with structured separation. In SmartEars (Raw), explanations collapse toward low reliability despite moderate-to-high fidelity. This indicates shortcut learning in noisy barn environments, where models rely on persistent background noise rather than disease-specific events. \texttt{AGRI-Fidelity} exposes this behavior, whereas masking metrics alone would suggest acceptable performance.

\noindent
\textbf{Preservation of Sparse, Time-Localized Biological Signals.}
On the SwineCough dataset, containing isolated transient cough events with minimal background noise, the model committee converges on time-localized regions, yielding reliability scores between 0.9 and 1.0. This empirically supports Theorem~2, demonstrating that \texttt{AGRI-Fidelity} preserves sparse, biologically meaningful signals without penalization.

\subsection{Consensus Model Performance Across Datasets}

\textbf{Poultry Variants.}
\textbf{Table~\ref{tab:all_models}} summarizes the four consensus architectures. On Chicken Denoised, LSTM achieves the highest accuracy (0.9615) and specificity (0.9998), followed by ResNet (0.9385), MLP (0.9096), and CNN (0.8654), reflecting temporal modeling benefits. On Chicken Spurious, CNN attains the highest accuracy (0.9865), while ResNet achieves the highest specificity (0.9964), indicating robustness to injected artifacts.

\noindent
\textbf{SmartEars.}
On SmartEars Clean, LSTM performs best (0.8413), followed by ResNet (0.8065), CNN (0.7399), and MLP (0.6710). Under SmartEars Noisy conditions, recall remains high but specificity declines, indicating bias from background noise. ResNet maintains the most balanced trade-off.

\noindent
\textbf{SwineCough.}
All models perform strongly on SwineCough, with ResNet achieving the highest accuracy (0.9817) and high specificity, consistent with cleaner acoustic conditions. \textbf{Figure~\ref{fig:exampleIGs}} shows representative IG maps across consensus models.

\begin{table}[t]
\centering
\caption{\textbf{Distribution of FDR Trend Types Across Class Labels:} For each class, we report the total number of samples and the count exhibiting Downward and Near-1 FDR trends. A Downward trend indicates rapid FDR decay with increasing consensus threshold, whereas a Near-1 trend indicates persistently high FDR across consensus levels.}
\label{tab:fdr_trend_distribution}
\scriptsize
\setlength{\tabcolsep}{6pt}
\begin{tabular}{@{}lccc@{}}
\toprule
Class & Total & Downward Trend & Near-1 Trend \\
\midrule
Healthy   & 139 & 120 & 19 \\
Unhealthy & 121 & 40  & 81 \\
\bottomrule
\end{tabular}
\end{table}

\begin{table}[t]
\centering
\caption{\textbf{Average Fidelity Trends Across Classes and Models:} For each class and trend type (Downward, Near-One), we report mean fidelity values for CNN, LSTM, MLP, and ResNet consensus models to enable cross-architecture comparison.}
\label{tab:fidelity_combined}
\scriptsize
\setlength{\tabcolsep}{4pt}
\begin{tabular}{@{}llcccc@{}}
\toprule
Class & Trend & CNN & LSTM & MLP & ResNet \\
\midrule
\multirow{2}{*}{Healthy}
  & Downward  & 0.1589 & 0.0206 & 0.0457 & 0.0159 \\
  & Near-One  & 0.121 & 0.0319 & 0.0137 & 0.0175 \\
\midrule
\multirow{2}{*}{Unhealthy}
  & Downward  & 0.5161 & 0.2194 & 0.1833 & 0.4263 \\
  & Near-One  & 0.7994 & 0.4519 & 0.7040 & 0.8542 \\
\bottomrule
\end{tabular}
\end{table}

\subsection{Consensus Stability and FDR Trend Analysis}

\textbf{Table~\ref{tab:fdr_trend_distribution}} summarizes FDR trend distributions across classes. 
Using the Near-1 trend to identify spurious behavior (Unhealthy as ground truth) yields 81 TP, 19 FP, 40 FN, and 120 TN, corresponding to 80.6\% precision, 66.9\% recall, 86.3\% specificity, and 76.5\% accuracy. These results show that the Near-1 FDR trend detects artifact-dominated explanations with high precision and specificity, while maintaining moderate recall. \textbf{Table~\ref{tab:fidelity_combined}} reports mean fidelity by class and FDR trend. For Unhealthy samples, fidelity remains higher than for Healthy samples across models and trends, indicating committee agreement on the injected continuous artifacts.

\noindent
\textbf{Computational Cost.}
\texttt{AGRI-Fidelity} operates post hoc on pre-trained models. Computing fidelity requires two forward passes per consensus model, while the permutation-based null construction has complexity $\mathcal{O}(B \cdot K \cdot T \cdot F)$. These operations are vectorized and efficient in practice. Although overall runtime exceeds lightweight masking metrics (AI, AD, AG), it remains substantially lower than ROAR, which retrains models for multiple masking thresholds ($\geq K/2$).

\section{Methodological Considerations and Design Rationale}

We outline the following six methodological considerations that clarify the statistical foundation and design principles of \texttt{AGRI-Fidelity}.  \textbf{(1) Choice of auditing explainer.} We adopt Integrated Gradients (IG) as the unified explainer across committee models. Unlike perturbation-based methods (e.g., LIME), which are sensitive to sampling variance and surrogate instability, IG satisfies completeness and implementation invariance. Thus, cross-model disagreement reflects differences in learned representations rather than explainer-induced noise. \noindent \textbf{(2) Explainer-agnostic evaluation.} The comparison between the IG-based committee and decoder- or NMF-based explainers is principled. \texttt{AGRI-Fidelity} evaluates the final $T \times F$ attribution mask independently of how it is generated. Whether derived from gradients, decoding, or surrogate fitting, the framework audits the statistical reliability of the resulting representation. \noindent \textbf{(3) FDR vs. overlap metrics.} We employ False Discovery Rate (FDR) rather than IoU, Precision, or Recall. Overlap-based measures ignore the null distribution and may reward agreement on stationary artifacts. By constructing a cyclic temporal permutation null, FDR quantifies whether consensus exceeds chance alignment, providing a statistically validated reliability criterion. \noindent \textbf{(4) Sampling-rate heterogeneity.} Farm sensors operate at heterogeneous sampling rates, making separate training pipelines impractical. We standardize recordings via resampling while recognizing that aggressive downsampling may suppress high-frequency disease signatures. \noindent \textbf{(5) Absence of spatial ground truth.} Bioacoustic datasets lack pixel-level annotations of disease-relevant regions in time–frequency space. In this setting, cross-model consensus under a permutation-based null serves as a conservative proxy for reliability without requiring explicit ground-truth masks. \noindent \textbf{(6) Continuous artifacts and OOD masking.} Naïve zero-masking in fidelity evaluation may introduce out-of-distribution artifacts. Using in-distribution imputation (e.g., diffusing salient bins to neighboring frequencies) reduces OOD distortions and better isolates genuine causal contributions while suppressing spurious correlations.

\section{Conclusions and Future Work}
\noindent \textbf{Summary of Contributions and Key Findings.}
This work introduces \texttt{AGRI-Fidelity}, a reliability-aware framework for evaluating explanations in bioacoustic disease detection. By modeling chance agreement via a cyclic-shift null, it separates time-localized physiological cues from stationary artifacts. 
Empirically, architectures capture distinct indicators.  \texttt{AGRI-Fidelity} highlights these consensus-supported regions. Although single-model fidelity magnitudes vary across seeds, committee-level directionality remains stable, demonstrating robust consensus evaluation.

\vspace{1mm}
\noindent \textbf{Limitations.}
Our null construction assumes spurious correlations manifest as continuous stationary artifacts, a reasonable assumption in barn environments. While cyclic temporal permutation suppresses such artifacts, it remains vulnerable to transient, time-localized shortcuts (e.g., sensor clicks or microphone disturbances). Because temporal shifting disrupts inter-model overlap, FDR may collapse and incorrectly validate these artifacts as reliable. Addressing non-stationary shortcuts requires extending the null model, for example through frequency-axis permutations or semantic anomaly detection. Another limitation is the lack of standardized spurious-correlation benchmarks in animal disease audio. Our synthetic setup injects high-frequency noise without modeling its statistical structure, enabling controlled validation but limited realism.

\vspace{1mm}
\noindent \textbf{Future Research Directions.}
Future work will develop statistically characterized artifact benchmarks to enable calibrated p-value estimation and principled rejection of unreliable consensus regions. We will extend \texttt{AGRI-Fidelity} with multi-axis null constructions and anomaly-aware modeling to improve robustness against non-stationary artifacts. Because farm sensors operate at varying sampling rates and resampling may discard high-frequency diagnostic cues, we will investigate adaptive multi-resolution modeling to preserve clinically meaningful spectral information while maintaining deployment practicality.

\newpage

\section{Generative AI Use Disclosure}
Generative AI tools were used in a limited capacity to assist with language editing, clarity, and polishing of the manuscript. All technical content, experimental design, theoretical development, analysis, and conclusions were conceived, implemented, and validated by the authors. No generative AI tool was used to produce substantive scientific contributions, generate experimental results, or make methodological decisions. All authors take full responsibility for the content of this paper and consent to its submission.

\bibliographystyle{IEEEtran}
\bibliography{mybib}

\end{document}

%% file: figures/proposed_approach_figure.tex
\usetikzlibrary{shapes.geometric, arrows.meta, positioning, calc, fit, backgrounds}


\resizebox{\textwidth}{!}{%

    \begin{tikzpicture}[
        >=Stealth,
        node distance=1.5cm and 2cm,
        box/.style={rectangle, draw=black!70, thick, rounded corners=3pt, minimum width=2.8cm, minimum height=1.2cm, align=center, fill=gray!10, font=\large},
        process/.style={rectangle, draw=black!70, thick, minimum width=3cm, minimum height=1.2cm, align=center, fill=blue!5, font=\large},
        metric/.style={rectangle, draw=black!70, thick, rounded corners=3pt, minimum width=3.5cm, minimum height=1.2cm, align=center, fill=green!10, font=\large},
        final/.style={rectangle, draw=black!70, thick, rounded corners=6pt, minimum width=3.5cm, minimum height=1.5cm, align=center, fill=purple!15, font=\textbf\large},
        arrow/.style={->, thick, draw=black!80},
        dashed_arrow/.style={->, thick, dashed, draw=black!80},
        group/.style={rectangle, draw=gray!60, thick, dashed, inner sep=10pt, rounded corners=5pt}
    ]

    \node[box] (input) {\textbf{Input Spectrogram}\\$X \in \mathbb{R}^{T \times F}$};

    \node[process, right=1.5cm of input, yshift=1.5cm] (m1) {Model $M_1$\\Explainer $\mathcal{E}(M_1, X)$};
    \node[right=3cm of input, font=\Large] (mdots) {$\vdots$};
    \node[process, right=1.5cm of input, yshift=-1.5cm] (mk) {Model $M_K$\\Explainer $\mathcal{E}(M_K, X)$};

    \draw[arrow] (input.east) -- (m1.west) node[midway, above left, font=\small, align=right] {Extract\\Top-$\kappa$};
    \draw[arrow] (input.east) -- (mk.west) node[midway, below left, font=\small, align=right] {Extract\\Top-$\kappa$};

    \node[group, fit=(m1) (mk) (mdots), label={[font=\large\bfseries, text=black!80]above:1. Committee Generation}] (committee) {};

    \node[process, right=1.7cm of m1, fill=orange!5] (sobs) {\textbf{Observed Consensus}\\$S_{obs}(t,f) = \frac{1}{K}\sum E_k(t,f)$};
    \node[process, right=1cm of sobs, fill=orange!5] (cobs) {\textbf{Stratification}\\$C_{\lambda}(t,f)$, $\lambda \in \Lambda$};

    \node[process, right=1.7cm of mk, fill=orange!5] (snull) {\textbf{Null Consensus} \\ 
    $E_k^{(b)}(t, f) = E_k((f + \Delta_k^{(b)}) \bmod T, t)$ \\[0.02in] $S_{null}^{(b)} = \frac{1}{K}\sum E_k^{(b)}$};
    \node[process, right=1cm of snull, fill=orange!5] (cnull) {\textbf{Stratification}\\$C_{\lambda}^{null, (b)}(t,f)$, $\lambda \in \Lambda$};

    \draw[arrow] (m1.east) -- (sobs.west) node[midway, above, font=\small] {Mask $E_1$};
    \draw[arrow, dashed] (mk.east) -- (sobs.west);
    \draw[arrow] (sobs.east) -- (cobs.west);

    \draw[arrow, dashed] (m1.east) -- (snull.west);
    \draw[arrow] (mk.east) -- (snull.west) node[midway, below, font=\small] {Mask $E_K$};
    \draw[arrow] (snull.east) -- (cnull.west);

    \node[group, fit=(sobs) (cobs) (snull) (cnull), label={[font=\large\bfseries, text=black!80]above:2. Conformal Consensus (Stability)}] (conformal) {};

    \node[metric, right=3.5cm of cobs, yshift=-0cm] (fdr) {\textbf{Empirical FDR}$(\lambda)$, $\lambda \in \Lambda$ \\ 
    \small{(Evaluates feature alignment)}};
    \node[metric, below=1.8cm of fdr] (fidelity) {\textbf{Fidelity Score}\\[1mm] \small{(Impact of feature on Predictions)}};

    \draw[arrow] (cobs.east) -- (fdr.west);
    \draw[arrow] (cnull.east) -- (fdr.south west);

    \draw[dashed_arrow] (mk.south) |- ++(0,-1.5) -| (fidelity.south) node[near start, above, font=\small] {Masks $E_k$};

    \node[final, right=0.2cm of fdr, yshift=-1.45cm] (agri) {\textbf{AGRI-Fidelity}: \\ Mean Fidelity,  Reliability Score};

    \draw[arrow] (fdr.east) -| (agri.north);
    \draw[arrow] (fidelity.east) -| (agri.south);

    \node[group, fit=(fdr) (fidelity) (agri), label={[font=\large\bfseries, text=black!80]above:3. Reliability Evaluation}] (evaluation) {};

    \end{tikzpicture}
}
